\documentclass[conference]{IEEEtran}
\IEEEoverridecommandlockouts

\usepackage{cite}
\usepackage{amsmath,amssymb,amsfonts}
\usepackage{graphicx}
\usepackage{textcomp}
\usepackage{xcolor}

\usepackage{algorithm}
\usepackage{algpseudocode}
\usepackage{amsthm}

\theoremstyle{definition}

\newtheorem{theorem}{Theorem}[section]
\newtheorem{lemma}[theorem]{Lemma}
\newtheorem{corollary}[theorem]{Corollary}

\DeclareMathOperator*{\argmin}{arg\,min}

\DeclareMathOperator{\tr}{tr}

\newcommand{\R}{\mathbb{R}}
\newcommand{\C}{\mathbb{C}}

\def\BibTeX{{\rm B\kern-.05em{\sc i\kern-.025em b}\kern-.08em
    T\kern-.1667em\lower.7ex\hbox{E}\kern-.125emX}}
\begin{document}

\title{When Learning Hurts: Fixed-Pole RNN for Real-Time Online Training
\thanks{\vspace{-0.1in} $^{*}$ A.~Morgan and U.~Khan contributed equally to this work. \vspace{-0.1in}}
}

\author{
\IEEEauthorblockN{Alexander~Morgan$^{*}$}
\IEEEauthorblockA{\textit{dept. of EECS, MIT}\\Cambridge, MA, USA  \vspace{-0.35in}} 

\and
\IEEEauthorblockN{Ummay~Sumaya~Khan$^{*}$}
\IEEEauthorblockA{\textit{dept. of ECE, Virginia Tech}\\Alexandria, VA, USA  \vspace{-0.35in}} 

\and
\IEEEauthorblockN{Lingjia~Liu}
\IEEEauthorblockA{\textit{dept. of ECE, Virginia Tech}\\Alexandria, VA, USA  \vspace{-0.35in}} 

\and
\IEEEauthorblockN{Lizhong~Zheng}
\IEEEauthorblockA{\textit{dept. of EECS, MIT}\\Cambridge, MA, USA \vspace{-0.35in} } 

}

\maketitle

\begin{abstract}
Recurrent neural networks (RNNs) can be interpreted as discrete-time state-space models, where the state evolution corresponds to an infinite-impulse-response (IIR) filtering operation governed by both feedforward weights and recurrent poles. While, in principle, all parameters including pole locations can be optimized via backpropagation through time (BPTT), such joint learning incurs substantial computational overhead and is often impractical for applications with limited training data.
Echo state networks (ESNs) mitigate this limitation by fixing the recurrent dynamics and training only a linear readout, enabling efficient and stable online adaptation. In this work, we analytically and empirically examine why learning recurrent poles does not provide tangible benefits in data-constrained, real-time learning scenarios. Our analysis shows that pole learning renders the weight optimization problem highly non-convex, requiring significantly more training samples and iterations for gradient-based methods to converge to meaningful solutions. Empirically, we observe that for complex-valued data, gradient descent frequently exhibits prolonged plateaus, and advanced optimizers offer limited improvement. In contrast, fixed-pole architectures induce stable and well-conditioned state representations even with limited training data. Numerical results demonstrate that fixed-pole networks achieve superior performance with lower training complexity, making them more suitable for online real-time tasks.
\end{abstract}


\section{Introduction}
Recurrent neural networks (RNNs) are a popular method for modeling sequential data with applications ranging from communications system design to audio and natural language processing~\cite{goodfellow2016deep}. Despite their representational power, traditional RNNs require extensive backpropagation through time (BPTT) for weight updates, which imposes significant computational burden and demands large volumes of labeled data, often unattainable in low-latency wireless settings~\cite{letter_ummay}. 

Echo State Networks (ESNs)~\cite{jaeger2001echo} have emerged as a lightweight alternative to conventional RNNs by eliminating the need to train recurrent connections. In this framework, the internal reservoir dynamics are governed by fixed, randomly initialized weights, and only the output layer undergoes supervised adaptation. This architectural decoupling substantially reduces computational burden, making ESNs well-suited for low-latency training in scenarios where labeled data and pilot resources are limited such as wireless signal detection and channel equalization~\cite{zhou2020learning, zhou2020rcnet, xu2022rc, xu2023detect}.

Recent developments have moved beyond random reservoir design by leveraging structural information from the physical domain to inform weight initialization. In particular, the studies in~\cite{jere2023channel, jere2023theoretical, 10974467} provide a formal analysis of reservoir computing applied to wireless channel equalization. These works show that when the recurrent weights are configured to mirror the characteristics of minimum-phase (MP) channels, it can achieve near-optimal bit-error rate (BER) performance, aligning with fundamental theoretical limits. 

However, no prior work provides a theoretical and empirical investigation into the fundamental limitations of learning recurrent dynamics in RNNs for low-latency signal processing tasks. In this work, we show that gradient descent (GD) based recurrent pole learning induces a highly non-convex and ill-conditioned optimization problem, and lack of convexity and unbounded condition numbers pose key challenges leading to optimization plateaus, even in linear settings. Experiments on wireless symbol detection validate the analysis, demonstrating that fixed-pole reservoir architectures exhibit stable training and competitive performance compared to trainable RNNs.

This paper is structured as follows: Section~\ref{sec:rnn_bptt} reviews the fundamentals of RNN training via BPTT. Section~\ref{sec:pole_learning_hard} analyzes the limitations of gradient-based pole learning in RNNs. Section~\ref{sec:result} presents empirical validation. Section~\ref{sec:conclusion} concludes the work.

\section{RNN pole learning via BPTT}
\label{sec:rnn_bptt}
Consider a linear single-input and single-output RNN with a single recurrent neuron and identity activation. Let $x[t] \in \mathbb{R}$ denote the input signal and $h[t] \in \mathbb{R}$ denote the hidden state of the neuron at time $t$. The linear state update equation is:
\begin{equation}
    h[t] = w_{\mathrm{rec}}\, h[t-1] + w_{\mathrm{in}}\, x[t],
\end{equation}
where $w_{\mathrm{rec}} \in \mathbb{R}$ is the recurrent feedback weight and $w_{\mathrm{in}} \in \mathbb{R}$ is the input coupling weight. Taking the $z$-transform under zero initial conditions yields
\begin{align}
    H(z) &= w_{\mathrm{rec}} z^{-1} H(z) + w_{\mathrm{in}} X(z), \notag\\
    \frac{H(z)}{X(z)} &= \frac{w_{\mathrm{in}}}{1 - w_{\mathrm{rec}} z^{-1}}.
\end{align}
Thus, each recurrent neuron realizes a first-order IIR filter with pole located at $w_{\mathrm{rec}}$ and feedforward gain $w_{\mathrm{in}}$.

We now extend this interpretation to a network of $N_h$ recurrent neurons. Let $\mathbf{x}[t] \in \mathbb{R}^{N_{\mathrm{in}}}$ denote the input vector at time $t$, and let $\mathbf{h}[t] \in \mathbb{R}^{N_h}$ denote the hidden state vector. The linear RNN dynamics become $\mathbf{h}[t] = \mathbf{W}_{\mathrm{rec}}\, \mathbf{h}[t-1] + \mathbf{W}_{\mathrm{in}}\, \mathbf{x}[t],$ where $\mathbf{W}_{\mathrm{rec}} \in \mathbb{R}^{N_h \times N_h}$ is the recurrent weight matrix and $\mathbf{W}_{\mathrm{in}} \in \mathbb{R}^{N_h \times N_{\mathrm{in}}}$ is the input weight matrix.

The output at time $t$ is computed as 
\begin{equation}
    \hat{\mathbf{y}}[t] = \mathbf{W}_{\mathrm{out}}\, \mathbf{h}[t],
\end{equation}
with output weight matrix, $\mathbf{W}_{\mathrm{out}} \in \mathbb{R}^{N_{\mathrm{out}} \times N_h}$. 

Given a target sequence $\mathbf{Y} \in \mathbb{R}^{N_{\mathrm{out}} \times T}$, the sequence-level loss is defined as $\mathcal{L}(\hat{\mathbf{Y}}, \mathbf{Y})$. In conventional RNN training, all parameters $\mathbf{W} \in \{ \mathbf{W}_{\mathrm{in}}, \mathbf{W}_{\mathrm{rec}}, \mathbf{W}_{\mathrm{out}} \}$ are optimized jointly using BPTT to learn the system dynamics. The gradient with respect to any weight $\mathbf{W}$ satisfies~\cite{lillicrap2019backpropagation}
\begin{equation}
    \frac{\partial \mathcal{L}}{\partial \mathbf{W}}
    =
    \sum_{t=0}^{T-1}
    \frac{\partial \mathcal{L}}{\partial \mathbf{h}[t]}
    \frac{\partial \mathbf{h}[t]}{\partial \mathbf{W}}.
\end{equation}
The backpropagated state gradients obey the recursion~\cite{werbos2002backpropagation}
\begin{equation}
    \frac{\partial \mathcal{L}}{\partial \mathbf{h}[t]}
    =
    \sum_{k=t}^{T-1}
    (\mathbf{W}_{\mathrm{rec}}^\top)^{k-t}
    \left(\mathbf{W}_{\mathrm{out}}^\top
    \frac{\partial \mathcal{L}}{\partial \hat{\mathbf{y}}[k]}\right),
\end{equation}

which reveals the dependence of gradient flow on the spectral structure of $\mathbf{W}_{\mathrm{rec}}$, resulting in repeated propagation of error gradients, amplifying or attenuating based on pole magnitudes. Under linearization, $\mathbf{W}_{\mathrm{rec}}$ is diagonalizable with eigenvalues $\{\lambda_i\}$, and the recurrent dynamics can be interpreted as a bank of parallel, non-interacting IIR filters with poles $\lambda_i$~\cite{jere2023channel}. Consequently, the gradient with respect to each pole scales as  
\vspace{-0.05in}
\begin{equation}
    \frac{\partial \mathcal{L}}{\partial \lambda_i}    \propto    \sum_{k=t}^{T-1} \lambda_i^{k-t},
    \vspace{-0.02in}
\end{equation}
which grows or decays exponentially depending on $|\lambda_i|$. When $|\lambda_i| > 1$, gradients explode; when $|\lambda_i| < 1$, gradients vanish. 
  
In contrast, ESNs circumvent this difficulty by fixing the recurrent dynamics. The recurrent matrix $\mathbf{W}_{\mathrm{rec}}$ is randomly initialized and its spectral radius satisfies $\rho(\mathbf{W}_{\mathrm{rec}}) < 1$, ensuring stability. The matrices $\mathbf{W}_{\mathrm{rec}}$ and $\mathbf{W}_{\mathrm{in}}$ are kept fixed, which effectively fixes the associated pole locations. Learning is restricted to the output layer, where the weights are obtained by solving a convex least‑squares optimization problem.

\vspace{-0.1in}
\begin{equation}
    \mathbf{W}_{\mathrm{out}}
    =
    \argmin_{\mathbf{W}}
    \sum_{t=0}^{T-1}
    \| \mathbf{y}[t] - \mathbf{W}_{out}\; \mathbf{h}[t] \|_2^2 = \mathbf{YH}^{\dagger},
\end{equation}

where, $(\cdot)^{\dagger}$ denotes the Moore–Penrose pseudo-inverse and $\mathbf{H}\in \mathbb{R}^{N_h \times T}$ is the state matrix. This decoupling of recurrent dynamics from training avoids pole optimization entirely, enabling stable and efficient optimization while preserving temporal dynamics with fixed poles.

\section{Fundamental Limits of Pole Learning}
\label{sec:pole_learning_hard}
In this section, we analytically show that in a noiseless setting where pole learning is identifiable, even though standard optimization formulations for learning pole locations should work in principle, they are inherently ill-conditioned and non-convex, so they often fail in practice.

Let $B = \{(b_k, \beta_k)\}_{k = 1}^K$, with $b_k, \beta_k \in \C$, be a set of parameters describing a causal impulse response of the form
\begin{equation*}
    h[n; B] = \begin{cases}
        \sum_{k = 1}^K b_k \beta_k^n & \text{ if } n \geq 0\\
        0 & \text{ if } n < 0
    \end{cases}.
\end{equation*}
Note that we consider $B$ to be unordered with respect to $k$.
Unless otherwise stated, whenever we work with such a set of parameters, we will assume them in fully simplified form.
In particular, we assume that poles are pairwise distinct, i.e.,
for $k \neq k'$, $\beta_k \neq \beta_{k'}$ and that coefficients are non-zero, i.e., $b_k \neq 0$ for all $k$.
Achieving this form is always possible by combining terms with the same pole and dropping terms with zero coefficient.
Given a collection of input-output pairs $(x^{(r)}[\cdot], y^{(r)}[\cdot])$ with $y^{(r)}[n] = (h[\cdot; B] * x^{(r)})[n]$, our objective is to recover the parameters $B$.
Our first goal is to characterize when this is possible to better understand the problem of learning pole locations $\beta_k$ from observed datasets.

\begin{lemma}\label{lemma:linear_indepedence}
    Fix any $K$ pairwise distinct but otherwise unrestricted $\beta_1, \dots, \beta_K \in \C$, and choose any $N \geq K$.
    Let $s_k[n]$ be the length-$N$ sequence defined by $s_k[n] = \beta_k^n$ for $0 \leq n < N$.
    Then, $s_k$ are linearly independent, i.e., if $b_1, \dots, b_K \in \C$ such that for all $0 \leq n < N$, $\sum_{k = 1}^K b_k s_k[n] = 0$, then $b_k = 0$ for all $1 \leq k \leq K$.
\end{lemma}
\begin{proof}
    Assume without loss of generality that $N = K$. Let
    \begin{equation*}
        \underbrace{\begin{bmatrix}
        1 & 1 & \cdots & 1 \\
        \beta_1 & \beta_2 & \cdots & \beta_K \\
        \vdots & \vdots & \ddots & \vdots \\
        \beta_1^{K-1} & \beta_2^{K-1} & \cdots & \beta_K^{K-1}
        \end{bmatrix}}_{\triangleq V^\top(\beta_1, \dots, \beta_K)}
        \begin{bmatrix}
            b_1\\
            b_2\\
            \vdots\\
            b_K
        \end{bmatrix}
        =
        \begin{bmatrix}
            0\\
            0\\
            \vdots\\
            0
        \end{bmatrix}.
    \end{equation*}
    Since $V$ is a Vandermonde matrix with
    \begin{equation*}
        \det(V^\top(\beta_1, \dots, \beta_K)) = \prod_{1 \leq k < k' \leq K} (\beta_{k'} - \beta_k) \neq 0,
    \end{equation*}
    we conclude that $0 = b_1 = \dots = b_K$.
\end{proof}
Next, we show that if the impulse responses for two parameter sets $B, \tilde{B}$ match for sufficiently many samples, then it implies $B = \tilde{B}$. In other words, given sufficiently many observations, it is in principle possible to determine the parameters describing the impulse response.
\begin{theorem}\label{theorem:can_recover}
    Let $K \geq 1$, and fix a set of parameters $B = \{(b_k, \beta_k)\}_{k = 1}^K$ with $b_k, \beta_k \in \C$. Suppose we observe the first $N \geq 2K$ samples $n = 0, \dots, N - 1$ from the causal impulse response $h[n; B] = \sum_{k = 1}^K b_k \beta_k^n$. Then, it is possible to recover the values $\{(b_k, \beta_k)\}_{k = 1}^K$.
    Formally, for any other alternative choice of parameters $\tilde{B} = \{(\tilde{b}_k, \tilde{\beta}_k)\}_{k = 1}^{\tilde{K}}$ with $\tilde{K} \leq K$, if
    \begin{equation*}
        h[n; B] = h[n; \tilde{B}]
    \end{equation*}
    for all $0 \leq n < N$, then $B = \tilde{B}$, and therefore $\tilde{K} = K$.
\end{theorem}
\begin{proof}
    Let $\tilde{B}^{\text{neg}} \triangleq \{(-\tilde{b}_k, \tilde{\beta}_k)\}_{k = 1}^{\tilde{K}}$, and consider a combined set of parameters $\{(g_{m}, \gamma_{m})\}_{m = 1}^{M} = B \cup \tilde{B}^{\text{neg}}$, where $M \leq N$ and the $\{\gamma_m\}$ are distinct, i.e., we combine terms with the same $\gamma_m$.
    However, coefficients $g_m$ may equal zero, so we temporarily disregard the full simplification assumption from earlier with respect to $\{(g_{m}, \gamma_{m})\}_{m = 1}^{M}$.
    Then,
    \begin{equation*}
        \eta[n] \triangleq h[n; B] - h[n; \tilde{B}] 
        = \sum_{m = 1}^M g_m \gamma_m^n
        = 0
    \end{equation*}
    for all $0 \leq n < N$. By Lemma \ref{lemma:linear_indepedence},
    \begin{equation*}
        0 = g_m = \left(\sum_{k : \beta_k = \gamma_m} b_k\right) - \left(\sum_{k : \tilde{\beta}_k = \gamma_m} \tilde{b}_k\right),
    \end{equation*}
    for all $m$. These sums both contain exactly one term, which must also be nonzero. It follows that $K = M = \tilde{K}$. For any $m$, $0 \neq \sum_{k : \beta_k = \gamma_m} b_k = \sum_{k : \tilde{\beta}_k = \gamma_m} \tilde{b}_k \triangleq w_m$, so $(w_m, \gamma_m) \in B \cap \tilde{B}$.
    Thus, $B = \tilde{B}$.

    Note that if $N < 2K$, then impulse response representations may not be unique. For instance, consider any pairwise distinct parameters $\{\beta_k\}_{k = 1}^K \cup \{\tilde{\beta}_k\}_{k = 1}^K$ , i.e., $|\{\beta_k\}_{k = 1}^K \cup \{\tilde{\beta}_k\}_{k = 1}^K| = 2K$. By dimensionality constraints, there must exist $b_k, \tilde{b}_k$ not all equal to $0$ (by Lemma \ref{lemma:linear_indepedence}, all $b_k, \tilde{b}_k \neq 0$ when $N = 2K - 1$) such that
    \begin{equation*}
        \eta[n] = h[n; B] - h[n; \tilde{B}] 
        = \sum_{k = 1}^K b_k \beta_k^n - \sum_{k = 1}^K \tilde{b}_k \tilde{\beta}_k^n\\
        = 0
    \end{equation*}
    for all $0 \leq n < N$. Thus, when $N$ is not sufficiently large, it is possible that $B \neq \tilde{B}$.
\end{proof}

Given at least $2K$ samples of the impulse response, it is theoretically possible to determine the system parameters, as Theorem \ref{theorem:can_recover} shows. All that is required is a thorough search over all possible system parameters, checking to see if the first $2K$ samples of the candidate impulse response matches the desired response. But it is in fact possible to determine the system parameters given at least $2K$ samples of any (reasonable) input-output example pair.

\begin{lemma}\label{lemma:recover_impulse_response}
    Fix any $N \geq 1$. Suppose a causal signal $x[n]$ is filtered using a causal impulse response $h[n]$ to produce $y[n]$, and we observe the first $N$ samples of this input-output pair. In particular, we have knowledge of $x[0], \dots, x[N - 1]$ and $y[0], \dots, y[N - 1]$. If $x[0] \neq 0$, then it is possible to recover $h[n]$ for $0 \leq n < N$. In particular, if $\tilde{h}$ is another causal impulse response with $y[n] = (\tilde{h} * x)[n]$ for all $0 \leq n < N$, then $\tilde{h}[n] = h[n]$ for all $0 \leq n < N$.
\end{lemma}
\begin{proof}
    Write
    \begin{align*}
        \underbrace{\begin{bmatrix}
        x_0 & 0 & \cdots & 0 \\
        x_1 & x_0 & \cdots & 0 \\
        \vdots & \vdots & \ddots & \vdots \\
        x_{N - 1} & x_{N - 2} & \cdots & x_0
        \end{bmatrix}}_{\triangleq L}
        \begin{bmatrix}
            h_0\\
            h_1\\
            \vdots\\
            h_{N - 1}
        \end{bmatrix}
        =
        \begin{bmatrix}
            y_0\\
            y_1\\
            \vdots\\
            y_{N - 1}
        \end{bmatrix}
    \end{align*}
    with $x_0 \triangleq x[0]$ and similarly for $y$ and $h$. The lower-triangular matrix $L$ has $\det(L) = (x[0])^n \neq 0$, so we can invert $L$ to find $h[0], \dots, h[N - 1]$ in terms of $x[0], \dots, x[N - 1]$ and $y[0], \dots, y[N - 1]$.
\end{proof}

Combining Theorem \ref{theorem:can_recover} and Lemma \ref{lemma:recover_impulse_response} yields
\begin{corollary}\label{corollary:recover_parameters_from_input_output_example}
    Let $K \geq 1$, and fix a set of parameters $B = \{(b_k, \beta_k)\}_{k = 1}^K$ with $b_k, \beta_k \in \C$. Suppose we observe the first $N \geq 2K$ samples $n = 0, \dots, N - 1$ of any input-output pair $(x[\cdot], y[\cdot])$ related by $y[n] = (h[\cdot; B] * x)[n]$. If $x[0] \neq 0$, then it is possible to recover $B$. Formally, the only choice of parameters $\tilde{B} = \{(\tilde{b}_k, \tilde{\beta}_k)\}_{k = 1}^K$ satisfying $y[n] = (h[\cdot; \tilde{B}] * x)[n]$ for $0 \leq n < N$ is $\tilde{B} = B$.
\end{corollary}
\begin{proof}
    By Lemma \ref{lemma:recover_impulse_response}, $y[n] = (h[\cdot; B] * x)[n]$ and $y[n] = (h[\cdot; \tilde{B}] * x)[n]$ for $0 \leq n < N$ together imply that $h[n; B] = h[n; \tilde{B}]$ for $0 \leq n < N$. Then, by Theorem \ref{theorem:can_recover}, $B = \tilde{B}$.
\end{proof}
The prior results demonstrate that learning system parameters is in principle possible under very mild assumptions. In particular, provided sufficiently many samples of any non-zero input-output pair, we can recover the system parameters by searching over all possible configurations and selecting the unique system that matches the observed data. However, this approach is not computationally feasible, especially in high dimensions, i.e., when $K$ becomes large.

These results will also show, in fact, that standard optimization objectives designed to recover these parameters will generally be non-convex, so in practice, recovery of system parameters is not straightforward. One trait that prevents convexity is lack of ordering among parameters, and pole locations in particular. Consider an unordered set of parameters $B = \{(b_k, \beta_k)\}_{k = 1}^K$ and an associated ordered set $B^{\text{ord}} = ((b_k, \beta_k))_{k = 1}^K$, which can be viewed as a vector in $\R^{2K}$. Standard optimization problems over ordered parameters in $\R^{2K}$, which is required in practical algorithms such as gradient descent, leads to non-convex objectives. In particular, 

\begin{lemma}\label{lemma:convex_property_failure}
    Fix $K \geq 1$ and $N \geq 2K$. Let $B = \{(b_k, \beta_k)\}_{k = 1}^K$ be a set of unordered parameters, let $B^{\text{ord}} = ((b_k, \beta_k))_{k = 1}^K \in \R^{2K}$ be any associated ordered set, and $\sigma(\cdot)$ be any non-identity permutation on $[K]$. Let $\sigma(B^{\text{ord}}) \triangleq \sigma\left(((b_k, \beta_k))_{k = 1}^K\right) \in \R^{2K}$. Then, for any $0 < \lambda < 1$, $\lambda B^{\text{ord}} + (1 - \lambda) \sigma(B^{\text{ord}})$ as an unordered set of parameters, does not equal $B$. In particular, with slight abuse of notation, $h[\cdot; \lambda B^{\text{ord}} + (1 - \lambda) \sigma(B^{\text{ord}})] \neq h[\cdot; B]$.
\end{lemma}
\begin{proof}
    The poles cannot remain in the same locations for $0 < \lambda < 1$. In particular,
    \begin{align*}
        \sum_k &|\lambda \beta_k + (1 - \lambda) \beta_{\sigma(k)}|^2\\
        &= \sum_k \lambda |\beta_k|^2 + (1 - \lambda) |\beta_{\sigma(k)}|^2 - \lambda(1 - \lambda) |\beta_k - \beta_{\sigma(k)}|^2\\
        &= \sum_k |\beta_k|^2 - \lambda(1 - \lambda) |\beta_k - \beta_{\sigma(k)}|^2\\
        &< \sum_k |\beta_k|^2.
    \end{align*}
    In fact, this proof can be used to show that the same result holds for other parameterizations. For example, if $\beta_k$ are expressed in polar coordinates, then we can still apply the prior analysis to the (real-valued) polar coordinates directly. 
\end{proof}

Under the assumptions and with the setup of Corollary \ref{corollary:recover_parameters_from_input_output_example} (importantly, we assume $N \geq 2K$), a typical objective used to optimize over ordered estimated parameters $\tilde{B}^{\text{ord}}$ is
\begin{equation*}
    \mathcal{L}(\tilde{B}^{\text{ord}}; x[\cdot], y[\cdot]) = d(y[\cdot], (h[\cdot; \tilde{B}^{\text{ord}}] * x)[\cdot])
\end{equation*}
where $d(\cdot, \cdot) \geq 0$ is a measure of distance between two length-$N$ vectors with the property that $d(a, b) = 0$ if and only if $a = b$, e.g., $d(a, b) = \|a - b\|_2^2$. We also allow $N = \infty$ for the purpose of analysis. Let $B^{\text{ord}}$ be an arbitrary ordering on $B$, and let $\sigma$ be an arbitrary non-identity permutation on $[K]$. Then, by Lemma \ref{lemma:convex_property_failure},
\begin{equation*}
    \mathcal{L}(B^{\text{ord}}; x[\cdot], y[\cdot]) = \mathcal{L}(\sigma(B^{\text{ord}}); x[\cdot], y[\cdot]) = 0,
\end{equation*}
but for any $0 < \lambda < 1$, 
\begin{equation*}
    \mathcal{L}(\lambda B^{\text{ord}} + (1 - \lambda)\sigma(B^{\text{ord}}); x[\cdot], y[\cdot]) > 0.
\end{equation*}

For any objective with the property that $\mathcal{L}(\tilde{B}^{\text{ord}}) = 0$ if and only if $\tilde{B} = B$ with $B$ being the ground-truth system, the same result holds, i.e., we lose convexity. Even though the system parameters may be in principle recoverable, it can be difficult in practice.

Assuming poles are now inside the unit circle, let
\vspace{-0.1in}
\begin{equation*}
    H(e^{i \omega}; B) = \sum_{k = 1}^K \frac{b_k}{1 - \beta_k e^{- i \omega}}
\end{equation*}
be the DTFT of $h[\cdot; B]$, and consider the objective
\begin{equation*}
    \mathcal{L}'(\tilde{B}^{\text{ord}}; B) = \frac{1}{2 \pi} \int_{-\pi}^\pi |H(e^{i \omega}; \tilde{B}^{\text{ord}}) - H(e^{i \omega}; B)|^2 \, d\omega
\end{equation*}
which has $\mathcal{L}'(\tilde{B}^{\text{ord}}; B) = 0$ if and only if $H(e^{i \omega}; \tilde{B}^{\text{ord}}) = H(e^{i \omega}; B)$. Thus, since the full impulse response can be recovered by an inverse transform, $\mathcal{L}'(\tilde{B}^{\text{ord}}; B) = 0$ if and only if $\tilde{B} = B$, so convexity over $\tilde{B}^{\text{ord}}$ does not hold. The same applies to the objective
\begin{equation*}
    \mathcal{L}''(\tilde{B}^{\text{ord}}; B) = \frac{1}{2 \pi} \int_{-\pi}^\pi \left|1 - \frac{H(e^{i \omega}; \tilde{B}^{\text{ord}})}{H(e^{i \omega}; B)}\right|^2 \, d\omega, 
\end{equation*}
which is a more direct consideration in channel equalization, and more generally, for any $w(\omega) > 0$ to
\begin{equation*}
    \mathcal{L}'''(\tilde{B}^{\text{ord}}; B) = \frac{1}{2 \pi} \int_{-\pi}^\pi w(\omega) |H(e^{i \omega}; \tilde{B}^{\text{ord}}) - H(e^{i \omega}; B)|^2 \, d\omega.
\end{equation*}

Next, we show that the objective $\mathcal{L}$ can have Hessian with unbounded condition number at an optimal point $\tilde{B}^{\text{ord}} = B^{\text{ord}}$, leading to ``narrow-valley'' behavior. For a symmetric square matrix $A$, we let $\kappa(A) = \lambda^{\text{max}} / \lambda^{\text{min}}$ denote the condition number of $A$. We will first consider a simple example where the condition number becomes unbounded. Then, we can extend this to the more general situation described above.

\begin{figure}[!ht]
    \centering
    \vspace{-0.1in}
    \includegraphics[width=0.5\linewidth]{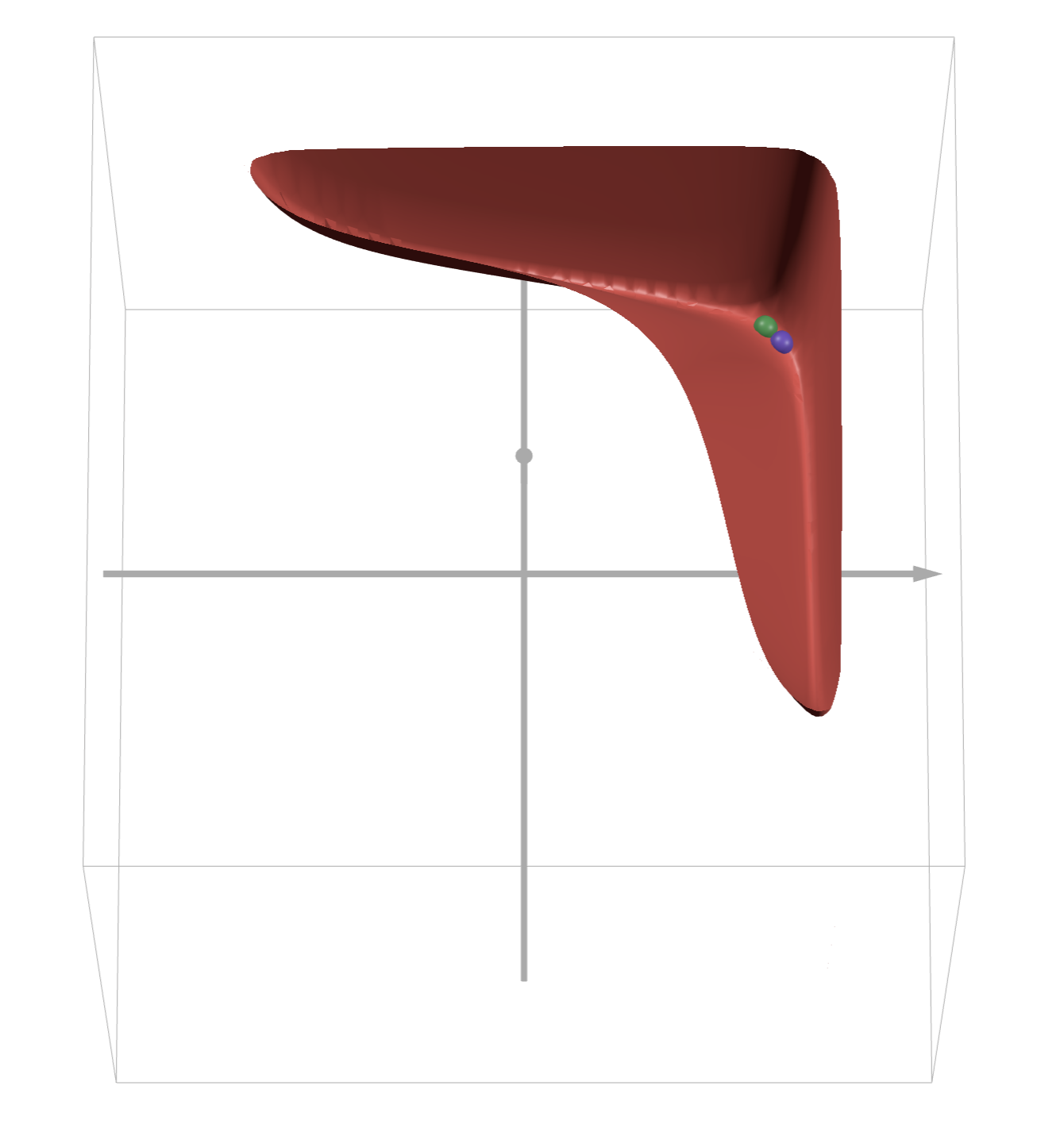}
    \vspace{-0.1in}

    \caption{Surface of $f(x, y)$ when $c = 0.75$ and $d = 0.8$. Green and purple points represent $(c, d)$ and $(d, c)$, respectively. Note the ``narrow-valley'' behavior of the objective at the optimum points, reflective of a large condition number for the Hessian.}
    \vspace{-0.2in}
    \label{fig:mesh1}
\end{figure}

\begin{figure*}[h]
    \centering
    \includegraphics[width=\linewidth]{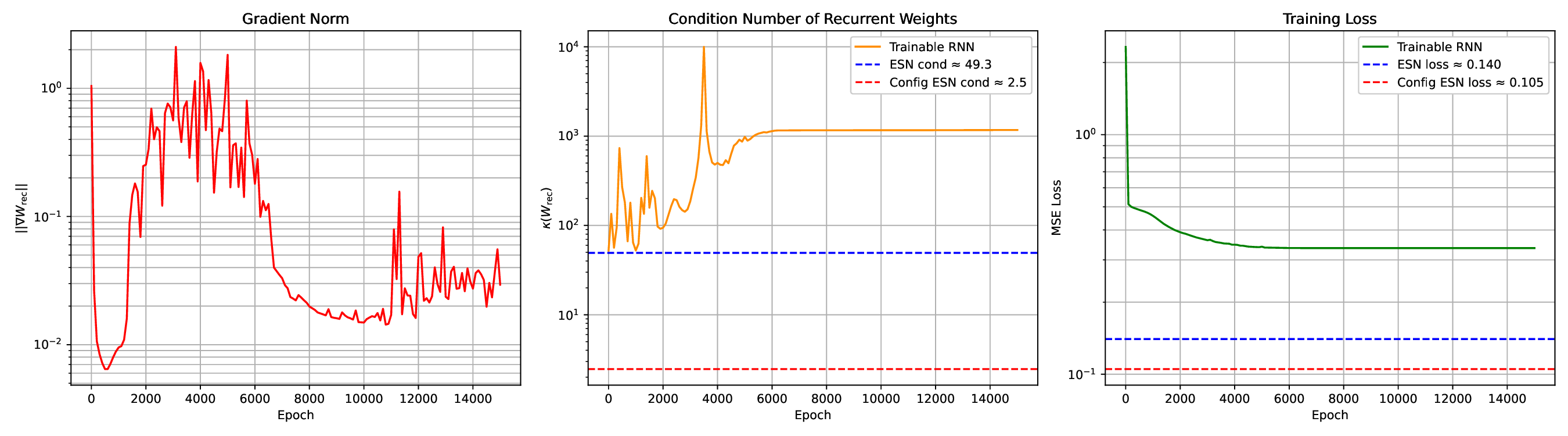}
    \vspace{-0.3in}
    \caption{Training dynamics of a trainable RNN. Left: gradient norm decay. Middle: condition number of recurrent matrix $W_{\mathrm{rec}}$. Right: training loss.}
    \vspace{-0.1in}
    \label{fig:training_dynamics}
\end{figure*}

\begin{figure}[h]
    \centering
    \vspace{-0.1in}
    \includegraphics[width=0.8\linewidth]{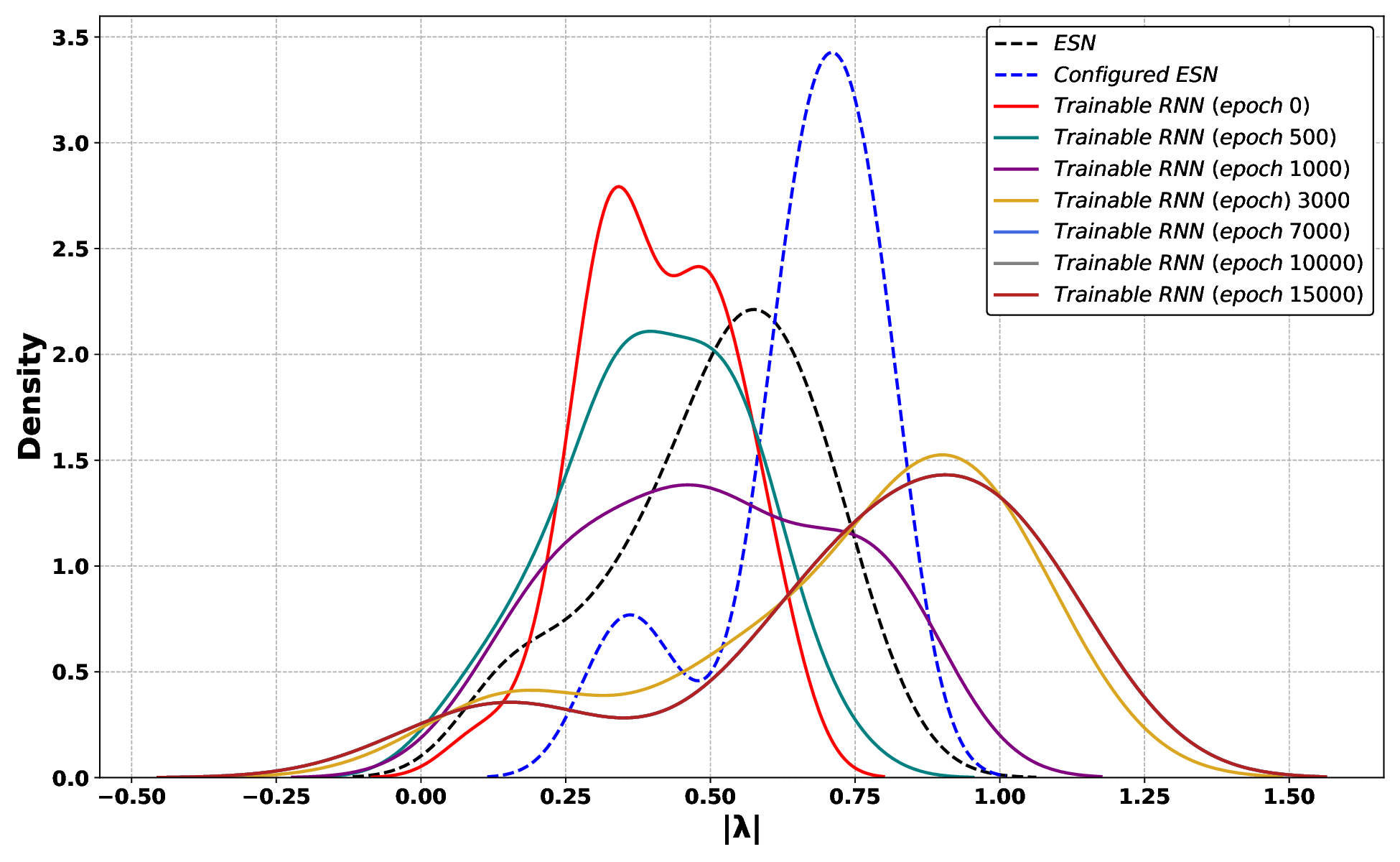}
    \vspace{-0.1in}
    \caption{Spectral density of eigenvalue magnitudes $|\lambda|$ over epochs.}
    \label{fig:spectral_density}
\end{figure}

Fix two real numbers $c, d \in (0, 1)$. Let $K = 2$ and suppose the ground-truth system is described by $B = \{(1, c), (1, d)\}$. Suppose we wish to learn the locations of the poles by minimizing the objective
\begin{align*}
    f(x, y) &= \sum_{n = -\infty}^\infty |h[n; B] - h[n; \tilde{B}]|^2\\
    &= \sum_{n = 0}^\infty \left(x^n + y^n - c^n - d^n\right)^2
\end{align*}
where $\tilde{B} = \{(1, x), (1, y)\}$ and $x, y \in (-1, 1) \subset \R$. The hessian of $f$ evaluated at an optimal point $(x, y) = (c, d)$ is
\begin{equation*}
    H_f(c, d) =
    \begin{bmatrix}
        \frac{2 \left(1 + c^2\right)}{\left(1 - c^2\right)^3} & \frac{2 \left(1 + c d\right)}{\left(1 - c d\right)^3}\\
        \frac{2 \left(1 + c d\right)}{\left(1 - c d\right)^3} & \frac{2 \left(1 + d^2\right)}{\left(1 - d^2\right)^3}
    \end{bmatrix}
    = 2 \begin{bmatrix}
        \langle u_c, u_c \rangle & \langle u_c, u_d \rangle\\
        \langle u_d, u_c \rangle & \langle u_d, u_d \rangle
    \end{bmatrix}
\end{equation*}
where $u_\alpha$ with $|\alpha| < 1$ is the square-summable sequence $u_\alpha[n] = (n + 1)\alpha^n$ for $n \geq 0$.
Let
\begin{align*}
    D = \det\left(\frac{1}{2}H_f(c, d)\right) = \langle u_c, u_c \rangle \langle u_d, u_d \rangle - \langle u_c, u_d \rangle^2 \geq 0
\end{align*}
with equality if and only if $c = d$
and
\begin{equation*}
    T = \tr\left(\frac{1}{2}H_f(c, d)\right) = \langle u_c, u_c \rangle + \langle u_d, u_d \rangle > 0.
\end{equation*}
We conclude that the hessian is positive definite if $c \neq d$. Furthermore, by continuity of second derivatives, the objective $f(x, y)$ must be (strongly) convex in a neighborhood of $(c, d)$ when $c \neq d$. The condition number is
\begin{equation*}
    \kappa(H_f(c, d)) = \frac{1 + \gamma}{1 - \gamma}, \quad with
\end{equation*}

\begin{equation*}
    \gamma = \sqrt{1 - \frac{4D}{T^2}} = \frac{\sqrt{\left(\langle u_c, u_c \rangle - \langle u_d, u_d \rangle\right)^2 + 4 \langle u_c, u_d \rangle^2}}{\langle u_c, u_c \rangle + \langle u_d, u_d \rangle}.
\end{equation*}
Since $D$ and $T$ are continuous in $c, d$, it follows that when $c \to d$, then $\gamma \to 1$ and therefore $\kappa(H_f(c, d)) \to \infty$. Even though the objective can be locally convex, the condition number can still become arbitrarily large. Furthermore, this behavior can occur without poles approaching unit magnitude.

This example is still relevant to the case with complex poles and gains, as the next lemma shows.

\begin{lemma}\label{lemma:hessian_condition_number_lower_bound}
    Consider a twice-differentiable function $f : \R^{m + n} \to \R$ at a point $z = (x, y)$ written in the form $f(x, y)$ for $x \in \R^m$ and $y \in \R^n$. Let $f_x(\cdot; y) : \R^m \to \R$ be the restriction of $f$ to the first argument with $y$ fixed, i.e., $f_x(x; y) = f(x, y)$. Then if $H_f(x, y) \succeq 0$,
    \begin{equation*}
        \kappa\left(H_f(x, y)\right) \geq \kappa\left(H_{f_x(\cdot; y)}(x)\right).
    \end{equation*}
\end{lemma}
\begin{proof}
    Let $u \in \R^{m + n}$ be a vector, and consider $u = (u_x, u_y)$ to be its components with $u_x \in \R^m$ and $u_y \in \R^n$. Then,
    \begin{align*}
        \lambda^{\text{max}}(H_f(x, y)) &= \max_{\|u\| = 1} u^\top H_f(x, y) u\\
        &\geq \max_{\|u\| = 1, u_y = 0} u^\top H_f(x, y) u\\
        &= \lambda^{\text{max}}\left(H_{f_x(\cdot; y)}(x)\right).
    \end{align*}
    Similarly,
    \begin{align*}
        \lambda^{\text{min}}(H_f(x, y)) &= \min_{\|u\| = 1} u^\top H_f(x, y) u\\
        &\leq \min_{\|u\| = 1, u_y = 0} u^\top H_f(x, y) u\\
        &= \lambda^{\text{min}}\left(H_{f_x(\cdot; y)}(x)\right).
    \end{align*}
    Combining these inequalities yields the desired result.
\end{proof}

The previously studied objective $f(x, y)$ is a restriction of the objective $\mathcal{L}$ used for learning complex-valued poles and gains. Regardless of cartesian/polar pole parametrization, the condition number $\kappa(H_f(c, d))$ remains a lower bound on the condition number of the full Hessian since $c, d \in (0, 1)$.

\section{Empirical Results}
\label{sec:result}
To empirically substantiate the theoretical difficulties inherent in pole learning, we consider an online real-time wireless symbol detection task. A RNN trained via BPTT is evaluated against two fixed-dynamics baselines: a randomly initialized ESN \cite{zhou2020learning} and a domain-informed, configured ESN \cite{10974467, letter_ummay}, each employing 32 recurrent neurons. To analyze the training behavior of these architectures, we track the evolution of gradient norms, condition numbers, and recurrent eigenvalue spectra across training epochs. These metrics provide insight into the internal optimization dynamics of trainable RNNs and enable a direct comparison fixed pole RNNs.

The Left panel of Fig.~\ref{fig:training_dynamics} shows the gradient norm $\|\nabla W_{rec}\|$ and after early fluctuations, the gradient magnitude stabilizes at a low but non-negligible value (around $10^{-1}$), indicating that gradients do not vanish entirely. However, despite this, the loss remains stagnant (Right panel), suggesting that the gradients are not effective in escaping flat or ill-conditioned regions of the landscape. This observation, together with the high condition number $\kappa(W_{\mathrm{rec}})$ (around $10^3$) seen in the middle plot, is consistent with the presence of saddle points and poor local curvature for gradient-based optimization as predicted by the Hessian analysis in Section~\ref{sec:pole_learning_hard}.
For the training loss plot over iterations, it is shown that there is a steep initial drop and then the loss quickly saturates, indicating that optimization is stuck in a narrow or flat region of the loss surface which again is reflecting the ill-conditioning caused by recurrent dynamics.

Fig.~\ref{fig:spectral_density} further illustrates how trainable RNNs suffer from uncontrolled drift in pole locations. We plot the spectral density of the magnitudes of eigenvalues $|\lambda|$ over various epochs. Unlike ESNs (dashed black) or configured RNNs (blue), where pole distributions are concentrated in a stable region, trainable RNNs show increasing dispersion, with eigenvalues spreading toward and beyond the unit circle. This drift reflects the absence of constraints on pole magnitudes during BPTT training and highlights the practical difficulty of maintaining stable dynamics in trainable RNNs.

\section{Conclusion}
\label{sec:conclusion}
Our analysis demonstrates that learning internal reservoir dynamics is difficult because training objectives can be highly non-convex and ill-conditioned, posing key challenges for GD-based optimization algorithms. Furthermore, these issues can occur even without poles approaching unit magnitude. This motivates the need for fixed randomly initialized internal dynamics or domain informed configured recurrent dynamics used in ESNs. These fixed reservoir-like configurations bypass the challenges of pole learning and lead to practical and efficient real-time training.

\bibliographystyle{IEEEtran}
\bibliography{ref}

\begin{thebibliography}{10}
\providecommand{\url}[1]{#1}
\csname url@samestyle\endcsname
\providecommand{\newblock}{\relax}
\providecommand{\bibinfo}[2]{#2}
\providecommand{\BIBentrySTDinterwordspacing}{\spaceskip=0pt\relax}
\providecommand{\BIBentryALTinterwordstretchfactor}{4}
\providecommand{\BIBentryALTinterwordspacing}{\spaceskip=\fontdimen2\font plus
\BIBentryALTinterwordstretchfactor\fontdimen3\font minus \fontdimen4\font\relax}
\providecommand{\BIBforeignlanguage}[2]{{%
\expandafter\ifx\csname l@#1\endcsname\relax
\typeout{** WARNING: IEEEtran.bst: No hyphenation pattern has been}%
\typeout{** loaded for the language `#1'. Using the pattern for}%
\typeout{** the default language instead.}%
\else
\language=\csname l@#1\endcsname
\fi
#2}}
\providecommand{\BIBdecl}{\relax}
\BIBdecl

\bibitem{goodfellow2016deep}
I.~Goodfellow, Y.~Bengio, A.~Courville, and Y.~Bengio, \emph{Deep learning}.\hskip 1em plus 0.5em minus 0.4em\relax MIT press Cambridge, 2016, vol.~1, no.~2.

\bibitem{letter_ummay}
U.~S. Khan, L.~Liu, S.~Jere, L.~Zheng, and Y.~Yi, ``Configuring rnn weights for mimo-ofdm receive processing: Informing rnn with domain knowledge,'' \emph{IEEE Wireless Commun. Letters}, pp. 1--1, 2025.

\bibitem{jaeger2001echo}
H.~Jaeger, ``The “echo state” approach to analysing and training recurrent neural networks-with an erratum note,'' \emph{Bonn, Germany: German national research center for information technology gmd technical report}, vol. 148, no.~34, p.~13, 2001.

\bibitem{zhou2020learning}
Z.~Zhou, L.~Liu, and H.-H. Chang, ``Learning for detection: Mimo-ofdm symbol detection through downlink pilots,'' \emph{IEEE Trans. Wireless Commun.}, vol.~19, no.~6, pp. 3712--3726, 2020.

\bibitem{zhou2020rcnet}
Z.~Zhou, L.~Liu, S.~Jere, Y.~Yi \emph{et~al.}, ``Rcnet: Incorporating structural information into deep rnn for mimo-ofdm symbol detection with limited training,'' \emph{arXiv preprint arXiv:2003.06923}, 2020.

\bibitem{xu2022rc}
J.~Xu, Z.~Zhou, L.~Li, L.~Zheng, and L.~Liu, ``Rc-struct: A structure-based neural network approach for mimo-ofdm detection,'' \emph{IEEE Trans. Wireless Commun.}, vol.~21, no.~9, pp. 7181--7193, 2022.

\bibitem{xu2023detect}
J.~Xu, L.~Li, L.~Zheng, and L.~Liu, ``Detect to learn: Structure learning with attention and decision feedback for mimo-ofdm receive processing,'' \emph{IEEE Trans. Commun.}, vol.~72, no.~1, pp. 146--161, 2023.

\bibitem{jere2023channel}
S.~Jere, R.~Safavinejad, L.~Zheng, and L.~Liu, ``Channel equalization through reservoir computing: A theoretical perspective,'' \emph{IEEE Wireless Commun. Letters}, vol.~12, no.~5, pp. 774--778, 2023.

\bibitem{jere2023theoretical}
S.~Jere, R.~Safavinejad, and L.~Liu, ``Theoretical foundation and design guideline for reservoir computing-based mimo-ofdm symbol detection,'' \emph{IEEE Trans. Commun.}, vol.~71, no.~9, pp. 5169--5181, 2023.

\bibitem{10974467}
S.~Jere, L.~Zheng, K.~Said, and L.~Liu, ``Towards xai: Configuring rnn weights using domain knowledge for mimo receive processing,'' \emph{IEEE Trans. Wireless Commun.}, pp. 1--1, 2025.

\bibitem{lillicrap2019backpropagation}
T.~P. Lillicrap and A.~Santoro, ``Backpropagation through time and the brain,'' \emph{Current opinion in neurobiology}, vol.~55, pp. 82--89, 2019.

\bibitem{werbos2002backpropagation}
P.~J. Werbos, ``Backpropagation through time: what it does and how to do it,'' \emph{Proceedings of the IEEE}, vol.~78, no.~10, pp. 1550--1560, 2002.

\end{thebibliography}


\end{document}